\documentclass[10pt,journal,compsoc]{IEEEtran}

\usepackage{epsfig}
\usepackage{graphicx}
\usepackage{amsmath}
\usepackage{amsthm}
\usepackage{amssymb}
\usepackage{verbatim}
\usepackage{booktabs}
\usepackage{amscd}
\usepackage{times}
\usepackage{subfigure} 
\usepackage{algorithm}
\usepackage{algorithmic}
\usepackage{hyperref}
\usepackage{color}

\newtheorem{thm}{Theorem}

\newtheorem{lem}[thm]{Lemma}
\newtheorem{prp}[thm]{Proposition}
\newtheorem{dfn}[thm]{Definition}
\newtheorem{prb}[thm]{Problem}

\newtheorem{remark}[thm]{Remark}
\newtheorem{example}[thm]{Example}
\newtheorem{que}[thm]{Question}

\def\A{\mathbb{A}}

\def\B{\boldsymbol{B}}

\def\X{\mathcal{X}}
\def\cS{\mathcal{S}}

\def\H{\mathcal{H}}
\def\I{\mathcal{I}}

\def\J{\mathcal{J}}
\def\Jf{\mathfrak{J}}
\def\Y{\mathcal{Y}}
\def\U{\mathcal{U}}
\def\bU{\boldsymbol{U}}
\def\btU{\tilde{\boldsymbol{U}}}
\def\Z{\mathcal{Z}}
\def\W{\mathcal{W}}

\def\0{\boldsymbol{0}}
\def\a{\boldsymbol{a}}
\def\b{\boldsymbol{b}}
\def\bt{\tilde{\b}}
\def\c{\boldsymbol{c}}
\def\g{\boldsymbol{\gamma}}

\def\s{\boldsymbol{s}}
\def\u{\boldsymbol{u}}

\def\y{\boldsymbol{y}}

\def\x{\boldsymbol{x}}
\def\xt{\tilde{x}}
\def\xtb{\boldsymbol{\xt}}
\def\e{\boldsymbol{e}}

\def\c{\boldsymbol{c}}
\def\bmu{\boldsymbol{\mu}}

\def\Re{\mathbb{R}}
\def\R{\mathcal{R}}
\def\Rt{\tilde{\mathcal{R}}}
\def\Pr{\mathbb{P}}
\def\transpose{\top}
\DeclareMathOperator*{\rank}{rank}
\DeclareMathOperator*{\codim}{codim}
\DeclareMathOperator*{\Var}{Var}
\DeclareMathOperator*{\Span}{Span}

\newcommand{\myparagraph}[1]{\smallskip\noindent\textbf{#1.}}

\newcommand{\be}{\begin{equation}}
\newcommand{\ee}{\end{equation}}
\newcommand{\bq}{\begin{eqnarray}}
\newcommand{\eq}{\end{eqnarray}}

\newcommand{\ba}{\left[ \begin{array}}
\newcommand{\ea}{\\ \end{array} \right]}

\begin{document} 

\title{Algebraic Clustering of Affine Subspaces}

\author{Manolis~C.~Tsakiris~     
        and~ Ren\'e~Vidal,~\IEEEmembership{Fellow,~IEEE}
\IEEEcompsocitemizethanks{\IEEEcompsocthanksitem The authors are with the Center of Imaging Science, Johns Hopkins University, Baltimore, MD, 21218, USA.\protect\\
E-mail: m.tsakiris@jhu.edu, rvidal@cis.jhu.edu
\IEEEcompsocthanksitem This work was partially supported by grants NSF 1447822 and 1218709, and ONR N000141310116.}}

\IEEEtitleabstractindextext{
\begin{abstract}
Subspace clustering is an important problem in machine learning with many applications in computer vision and pattern recognition. Prior work has studied this problem using algebraic, iterative, statistical, low-rank and sparse representation techniques. While these methods have been applied to both linear and affine subspaces, theoretical results have only been established in the case of linear subspaces.
For example, algebraic subspace clustering (ASC) is guaranteed to provide the correct clustering when the data points are in \emph{general position} and the union of subspaces is \emph{transversal}. In this paper we study in a rigorous fashion the properties of ASC in the case of affine subspaces. Using notions from algebraic geometry, we prove that the \emph{homogenization trick}, which embeds points in a union of affine subspaces into points in a union of linear subspaces, preserves the general position of the points and the transversality of the union of subspaces in the embedded space, thus establishing the correctness of ASC for affine subspaces.
\end{abstract}
}

\maketitle
\IEEEpeerreviewmaketitle

\section{Introduction}
Subspace clustering is the problem of clustering a collection of points drawn approximately from a union of linear or affine subspaces. This is an important problem in machine learning with many applications in computer vision and pattern recognition such as clustering faces, digits, images and motions \cite{Vidal:SPM11-SC}. Over the past $15$ years, a variety of subspace clustering methods have appeared in the literature, including iterative \cite{Bradley:JGO00,Tseng:JOTA00}, probabilistic \cite{Tipping-mixtures:99}, algebraic \cite{Vidal:PAMI05}, spectral \cite{Chen:IJCV09,Heckel:arxiv13}, and self-expressiveness-based \cite{Vidal:PRL14,Liu:TPAMI13,Elhamifar:TPAMI13,You:CVPR16-SSCOMP,Lu:ECCV12} approaches. Among them, the \emph{Algebraic Subspace Clustering} (ASC) method of \cite{Vidal:PAMI05}, also known as GPCA, establishes an interesting connection between machine learning and algebraic geometry (see also \cite{ICML2013_livni13} for another such connection). By describing a union of $n$ linear subspaces as the zero set of a system of homogeneous polynomials of degree $n$, ASC clusters the subspaces in closed form via polynomial fitting and differentiation (or alternatively by polynomial factorization \cite{Vidal:CVPR03-gpca}). 

\myparagraph{Merits of algebraic subspace clustering} 
In addition to providing interesting algebraic geometric insights in the problem of subspace clustering, ASC is unique among other methods in that it is guaranteed to provide the correct clustering, under the mild hypothesis that the union of subspaces is \emph{transversal} and the data points are in \emph{general position} (in an algebraic geometric sense). This entails that ASC can handle subspaces of dimension $d$ comparable to the ambient dimension $D$ (high relative dimension $d/D$). In contrast, most state-of-the-art methods, such as \emph{Sparse Subspace Clustering} (SSC) \cite{Elhamifar:TPAMI13,Soltanolkotabi:AS14}, require the subspaces to be of sufficiently small relative dimension. Therefore, instances of applications where ASC is a natural candidate, while SSC is in principle inapplicable, are projective motion segmentation \cite{Vidal:IJCV06-multibody}, 3D point cloud analysis \cite{Sampath:2010segmentation} and hybrid system identification \cite{Vidal:Automatica08}. Moreover, it was recently demonstrated in \cite{Tsakiris:FSASCICCV15} that, using the idea of \emph{filtrations} of unions of subspaces \cite{Tsakiris:Asilomar14,Tsakiris:SIAM17}, ASC not only can be robustified to noise, but also outperforms state-of-the-art methods in the popular benchmark dataset Hopkins155 \cite{Tron:CVPR07} for real world motion segmentation. 

\myparagraph{Dealing with affine subspaces} In several important applications, such as motion segmentation, the underlying subspaces do not pass through the origin, i.e., they are affine. 
Methods such as K-subspaces \cite{Bradley:JGO00,Tseng:JOTA00} and mixtures of probabilistic PCA \cite{Tipping-mixtures:99} can trivially handle this case by explicitly learning models of affine subspaces. Likewise, the spectral clustering method of \cite{Zhang:IJCV12} can handle affine subspaces by constructing an affinity that depends on the distance from a point to a subspace. However, these methods do not come with theoretical conditions under which they are guaranteed to give the correct clustering.
On the other hand, when data $\X=[\x_1\cdots \x_N ] \subset \Re^D$ lying in a union of $n$ distinct affine subspaces are embedded into \emph{homogeneous coordinates}  
\begin{align}
\tilde{\X} =\begin{bmatrix}
1 & 1 & \cdots & 1 \\
\x_1 & \x_2 & \cdots & \x_N
\end{bmatrix} \subset \Re^{D+1}, \label{eq:embeddingIntro}
\end{align} they lie in a union of $n$ distinct linear subspaces of $\Re^{D+1}$. If the \emph{embedded data} $\tilde{\X}$ satisfy the geometric separation conditions of \cite{Soltanolkotabi:AS14} with respect to the underlying union of linear subspaces, then SSC \cite{Elhamifar:TPAMI13} applied to $\tilde{\X}$ is guaranteed to yield a \emph{subspace preserving} affinity. While the conditions in \cite{Soltanolkotabi:AS14} have a clear geometric interpretation for linear subspaces, it is unclear what these conditions entail when applied to affine subspaces via the embedding in \eqref{eq:embeddingIntro}. On the other hand, recent work \cite{Wang:ICAIS16,Yang:ECCV16} shows that an $\ell_0$ version of SSC  ($\ell_0$-SSC) yields the correct clustering under mild conditions of general position (in a linear algebraic sense), and this analysis can be easily extended to affine subspaces. As opposed though to the polynomial complexity of $\ell_1$-SSC \cite{Elhamifar:TPAMI13}, $\ell_0$-SSC has exponential complexity. While the ASC approach discussed here also has exponential complexity in general, under certain conditions it is more efficient than $\ell_0$-SSC.\footnote{The worst-case complexity of ASC is $\mathcal{O}\left(N {n+D \choose n}^2 \right)$, which is linear in the number of data points, $N$, and exponential in the number of subspaces, $n$, and the dimension of the original data $D$. In contrast, the worst-case complexity of $\ell_0$-SSC is $\mathcal{O}\left(N(D+1)(d+1)^2{N-1 \choose d+1}\right)$, where $d$ is the maximal dimension of the affine subspaces. Hence, when $n$ and $D$ are small and $d \approx \mathcal{O}( D )$, the complexity of $\ell_0$-SSC as a function of N dominates that of ASC. A detailed comparison of the complexities of ASC and $\ell_0$-SSC is beyond the scope of the paper and is thus omitted.}

Returning to ASC, the traditional way to handle points from a union of affine subspaces \cite{Vidal:PhD03} is to use homogeneous coordinates as in \eqref{eq:embeddingIntro}, and subsequently apply ASC to the embedded data. We refer to this two-step approach as Affine ASC (AASC). Although AASC has been observed to perform well in practice, it lacks a sufficient theoretical justification. On one hand, while it is true that the embedded points live in a union of associated linear subspaces, 
it is obvious that they have a very particular structure inside these subspaces. Specifically, even if the original points are \emph{generic}, in the sense that they are sampled uniformly at random from the affine subspaces, the embedded points are clearly \emph{non-generic}, in the sense that they always lie in the zero-measure intersection of the union of the associated linear subspaces with the hyperplane $x_0=1$.\footnote{Here and in the rest of the paper, we consider only the uniform measure.}
Thus, even in the absence of noise, one may wonder whether this \emph{non-genericity} of the embedded points will affect the behavior of AASC and to what extent. Moreover, even if the affine subspaces are transversal, there is no guarantee that the associated linear subspaces are also transversal. Thus, it is natural to ask for conditions on the affine subspaces and the data points under which AASC is guaranteed to give the correct clustering.


\myparagraph{Paper contributions}
In this paper we adapt abstract notions from algebraic geometry to the context of unions of affine subspaces in order to rigorously prove the correctness of AASC in the absence of noise. More specifically, we define in a very precise fashion the notion of points being in \emph{general position} in a union of $n$ linear or affine subspaces. Intuitively, points are in general position if they can be used to uniquely reconstruct the union of subspaces they lie in by means of polynomials of degree $n$ that vanish on the points. Then we show that the embedding \eqref{eq:embeddingIntro} preserves the property of points being in general position, which is one of the two success conditions of ASC. We also show that the second condition, which is the \emph{transversality} of the union of linear subspaces in $\Re^{D+1}$ that is associated to the union of affine subspaces in $\Re^D$ under the embedding \eqref{eq:embeddingIntro}, is also satisfied, provided that 
 \begin{enumerate}
 	\item the union of subspaces formed by the linear parts of the affine subspaces is transversal, and
 	\item the translation vectors of the affine subspaces do not lie in the zero measure set of a certain algebraic variety.
 \end{enumerate} 

Our exposition style is for the benefit of the reader unfamiliar with algebraic geometry. We introduce notions and notations as we proceed and give as many examples as space allows. We leave the more intricate details to the various proofs.

\section{Algebraic Subspace Clustering Review} \label{section:ASC-overview}


This section gives a brief review of the ASC theory (\cite{Vidal:PAMI05,Derksen:JPAA07,Ma:SIAM08,Tsakiris:SIAM17}). After defining the subspace clustering problem in Section \ref{subsection:SubspaceClustering}, we describe unions of linear subspaces as algebraic varieties in Section \ref{subsection:UnionsLinearVarieties}, and give the main theorem of ASC (Theorem \ref{thm:ASC}) in terms of vanishing polynomials in Section \ref{subsection:MainTheoremASC}. In Section \ref{subsection:transversality} we elaborate on the main hypothesis of Theorem \ref{thm:ASC}, the \emph{transversality} of the union of subspaces. In Section \ref{subsection:FiniteCollection} we introduce  the notion of points \emph{in general position} (Definition \ref{dfn:DensePointsLinear}) and adapt Theorem \ref{thm:ASC} to the more practical case of a finite set of points (Theorem \ref{thm:ASCfinite}). 

\subsection{Subspace Clustering Problem} \label{subsection:SubspaceClustering}
Let $\X = \left\{\x_1,\dots,\x_N\right\}$ be a set of points that lie in an unknown union of $n>1$ linear subspaces $\Phi=\bigcup_{i=1}^n\cS_i$, where $\cS_i$ a linear subspace of $\Re^D$ of dimension $d_i <D$. The goal of subspace clustering is to find the number of subspaces, their dimensions, a basis for each subspace, and cluster the data points based on their subspace membership, i.e., find a decomposition or \emph{clustering} of $\X$ as $\X = \X_1 \cup \cdots \cup \X_n$, where $\X_i = \X \cap \cS_i$. 

\subsection{Unions of Linear Subspaces as Algebraic Varieties} \label{subsection:UnionsLinearVarieties}
The key idea behind ASC is 
that a union of $n$ linear subspaces $\Phi=\bigcup_{i=1}^n\cS_i$ of $\Re^{D}$ is the \emph{zero set} of a finite set of homogeneous\footnote{A polynomial in many variables is called homogeneous if each of its monomials has the same degree. For example, $x_1^2+x_1x_2$ is homogeneous of degree $2$, while $x_1^2+x_2$ is non-homogeneous of degree $2$.} polynomials of degree $n$ with real coefficients in $D$ indeterminates 
$x:=\left[x_1,\dots,x_D\right]^\transpose$. Such a set is called an \emph{algebraic variety}~\cite{AtiyahMacDonald-1994,Cox:2007}. For example, a union of $n$ hyperplanes $\Phi = \H_1 \cup \cdots \cup \H_n$, where the $i$th hyperplane $\H_i = \{ \x: \b_i^\transpose\x = 0\}$ is defined by its normal vector $\b_i\in\Re^D$, is the zero set of the polynomial 
\begin{equation}
\label{eq:hyperplanes}
p(x) = (\b_1^\top x) ( \b_2^\top x) \cdots (\b_n^{\transpose} x),
\end{equation} in the sense that a point $\x$ belongs to the union $\Phi$ if and only if $p(\x)=0$.
Likewise, the union of a plane with normal $\b$ and a line with normals $\b_1,\b_2$ in $\Re^3$ is the zero set of the two polynomials 
\begin{align}
p_1(x) = (\b^{\transpose} x) (\b_1^{\transpose} x)
~~\text{and}~~
p_2(x) = (\b^{\transpose} x) (\b_2^{\transpose} x).
\end{align} 
More generally, for $n$ subspaces of arbitrary dimensions, these \emph{vanishing polynomials} are homogeneous of degree $n$. Moreover, they are factorizable into $n$ linear forms, with each linear form defined by a vector orthogonal to one of the $n$ subspaces.\footnote{Strictly speaking this is not always true; it is true though in the generic case, for example, if the subspaces are transversal (see Definition \ref{dfn:transversal}).}

\subsection{Main Theorem of ASC} \label{subsection:MainTheoremASC}
The set $\I_\Phi$ of polynomials that vanish at every point of a union of linear subspaces $\Phi$ has a special
algebraic structure: it is closed under addition and it is closed under multiplication by any element of the \emph{polynomial ring}
$\R=\Re[x_1,\dots,x_D] $. Such a set of polynomials is called an \emph{ideal} \cite{AtiyahMacDonald-1994,Cox:2007} of $\R$. If we restrict our attention to the subset 
$\I_{\Phi,n}$ of $\I_\Phi$ that consists only of vanishing polynomials of degree $n$, we notice that
$\I_{\Phi,n}$ is a finite dimensional real vector space, because it is a subspace of $\R_n$, the latter being the set of all homogeneous polynomials of $\R$ of degree $n$, which is a vector space of dimension $M_n(D):={n+D-1 \choose n}$.

\begin{thm}[Main Theorem of ASC, \cite{Vidal:PAMI05}] \label{thm:ASC}
Let $\Phi = \bigcup_{i=1}^n\cS_i$ be a transversal union of linear subspaces of
$\Re^D$. Let $p_1,\dots,p_s$ be a basis for $\I_{\Phi,n}$
and let $\x_i$ be a point in $\cS_i$ such that $\x_i \not\in \bigcup_{i' \neq i} \cS_{i'}$.
Then $\cS_i = \Span( \nabla p_1|_{\x_i}, \dots,  \nabla p_s|_{\x_i})^{\perp}$.
\end{thm}
In other words, we can estimate the subspace $\cS_i$ passing through a point $\x_i$, as the orthogonal complement of the span of the gradients of all the degree-$n$ vanishing polynomials evaluated at $\x_i$.
Observe that the only assumption on the subspaces required by Theorem \ref{thm:ASC},
is that they are \emph{transversal}, a notion explained next. 

\subsection{Transversal Unions of Linear Subspaces}\label{subsection:transversality}

Intuitively, transversality is a notion of general position of subspaces, which entails that all intersections among subspaces are as small as possible, as allowed by their dimensions. Formally:
\begin{dfn}[\cite{Derksen:JPAA07}] \label{dfn:transversal}
A union $\Phi = \bigcup_{i=1}^n\cS_i$ of linear subspaces of $\Re^D$ is transversal, if for any subset $\Jf$ of $[n]:=\left\{1,2,\dots,n\right\}$
\begin{align}
\codim \left( \bigcap_{i \in \Jf} \cS_i\right) = 
\min\left\{D,\sum_{i \in \Jf} \codim(\cS_i)\right\},
\end{align} where $\codim(\cS)=D-\dim(\cS)$ denotes the codimension of $\cS$.
\end{dfn}  To understand Definition \ref{dfn:transversal}, let $\B_i$ be a $D \times c_i$ matrix containing a basis for $\cS_i^\perp$, where $c_i$ is the codimension of $\cS_i$, and let $\Jf$ be a subset of $[n]$, say $\Jf = \left\{1,\dots,\ell \right\}, \, \ell \le n$. 
Then a point $\x$ belongs to $\bigcap_{i \in \Jf} \cS_i$ if and only if $\x^\transpose \B_{\Jf} = \0$, where $\B_{\Jf}= \left[\B_1,\dots,\B_{\ell} \right]$. Hence, the dimension of $\bigcap_{i \in\Jf} \cS_i$ is equal to the dimension of the left nullspace of $\B_{\Jf}$, or equivalently, 
\begin{align}
\codim \left( \bigcap_{i \in \Jf} \cS_i\right) = \rank(\B_{\Jf}).
\end{align} Since $\B_{\Jf}$ is a $D \times \sum_{i \in \Jf}c_i$ matrix, we must have that
\begin{align}
\rank(\B_{\Jf}) \le \min\left\{D,\sum_{i \in \Jf} c_i\right\}.
\end{align} Hence, transversality is equivalent to $\B_{\Jf}$ being full-rank, as $\Jf$ ranges over all subsets of $[n]$, and $\B_{\Jf}$ drops rank if and only if all maximal minors of $\B_{\Jf}$ vanish, in which case there are certain algebraic relations between the basis vectors of $\cS_i^\perp, \, i \in \Jf$. Since any set given by algebraic relations has measure zero \cite{Burgisser:Springer2013}, this shows that a union of subspaces is transversal with probability 1.

\begin{prp}\label{prp:TransversalGeneral}
Let $\Phi = \bigcup_{i=1}^n \cS_i$ be a union of $n$ linear subspaces in $\Re^D$ of codimensions $0<c_i<D, \, i \in [n]$. Let $\b_{i1},\dots,\b_{ic_i}$ be a basis for $\cS_i^\perp$. If the vectors $\{\b_{i j_i}\}_{i=1,\dots,n}^{j_i=1,\dots,c_i}$ do not lie in the zero-measure set of a (proper) algebraic variety of $\Re^{D \times \sum_{i \in [n]} c_i}$, then $\Phi$ is transversal.
\end{prp}

 \begin{example}Consider two planes $\cS_1,\cS_2$ in $\Re^3$ with normals $\b_1$ and $\b_2$. Then one expects their intersection $\cS_1 \cap \cS_2$ to be a line, and hence be of codimension $2 = \min(3,1+1)$, unless the two planes coincide, which happens only if $\b_1$ is colinear with $\b_2$. Clearly, if one randomly selects two planes in $\Re^3$, the probability that they are not transversal is zero. If we consider a third plane $\cS_3$ with normal $\b_3$ such that every intersection $\cS_1 \cap \cS_2$, $\cS_1 \cap \cS_3$ and $\cS_2 \cap \cS_3$ is a line, then the three planes fail to be transversal only if $\cS_1 \cap \cS_2 \cap \cS_3$ is a line. But this can happen only if the three normals $\b_1,\b_2,\b_3$ are linearly dependent, which again is a probability zero event if the three planes are randomly selected. 
 \end{example}
This reveals the important fact that the theoretical conditions for success of ASC (in the absence of noise) are much weaker than those for other methods such as SSC, since as we just pointed out ASC will succeed almost surely (Theorem \ref{thm:ASC}).\footnote{Of course, the main disadvantage of ASC with respect to SSC is its exponential computational complexity, which remains an open problem.}

\subsection{Points In General Position} \label{subsection:FiniteCollection}
In practice, we may not be given the polynomials $p_1,\dots,p_s$ that vanish on a union of subspaces $\Phi=\bigcup_{i=1}^n\cS_i$, but rather a finite collection of points $\X=\left\{\x_1,\dots,\x_N\right\}$ sampled from $\Phi$. If we want to fully characterize $\Phi$ from $\X$, the least we can ask is that $\X$ uniquely defines $\Phi$ as a set, otherwise the problem becomes ill-posed. Since it is known that $\Phi$ is the zero set of $\I_{\Phi,n}$ \cite{Vidal:PAMI05}, i.e., $\Phi = \Z(\I_{\Phi,n})$, it is natural to require that $\Phi$ can be recovered as the zero set of all homogeneous polynomials of degree $n$ that vanish on $\X$.

\begin{dfn}[\bf Points in general position]
\label{dfn:DensePointsLinear}
	Let $\Phi$ be a union of $n$ linear subspaces of $\Re^{D}$,  and $\X$ a finite set of points in $\Phi$. We will say that $\X$ is in general position in $\Phi$, if $\Phi = \Z(\I_{\X,n})$.
\end{dfn} 

Recall from Theorem \ref{thm:ASC} that for ASC to succeed, we need a basis $p_1,\dots,p_s$ for $\I_{\Phi,n}$. The next result shows that if $\X$ is in general position in $\Phi$, then we can compute such a basis form $\X$. 

\begin{prp}\label{prp:LinearGeneralPositionEquivalence}
	$\X$ is in general position in $\Phi$ $\Leftrightarrow$ $\I_{\X,n} = \I_{\Phi,n}$.
\end{prp}
\begin{proof}
	$(\Rightarrow)$ Suppose $\X$ is in general position in $\Phi$, i.e., $\Phi = \Z(\I_{\X,n})$. We will show that $\I_{\X,n} = \I_{\Phi,n}$. The inclusion $\I_{\X,n} \supset \I_{\Phi,n}$ is immediate, since if $p \in \I_{\Phi,n}$ vanishes on $\Phi$, then it will vanish on the subset $\X$ of $\Phi$. Conversely, let $p \in \I_{\X,n}$. Since by hypothesis $\Phi = \Z(\I_{\X,n})$, we will have that $p(\x) = 0, \, \forall \x \in \Phi$, i.e., $p$ vanishes on $\Phi$, i.e., $p \in \I_{\Phi,n}$, i.e., $\I_{\X,n} \subset \I_{\Phi,n}$. 
	
	$(\Leftarrow)$ Suppose $\I_{\X,n} = \I_{\Phi,n}$; then $\Z(\I_{\X,n}) = \Z(\I_{\Phi,n})$. Since $\Phi = \Z(\I_{\Phi,n})$ \cite{Vidal:PAMI05}, we have $\Phi = \Z(\I_{\X,n})$.
\end{proof}

\noindent Next, we show that points in general position always exist.
\begin{prp}\label{thm:AlgebraicSampling}
	Any union $\Phi$ of $n$ linear subspaces of $\Re^D$ admits a finite subset $\X$ that lies in general position in $\Phi$. 
\end{prp} 
\begin{proof}
	This follows from Theorem 2.9 in \cite{Ma:SIAM08}, together with the regularity result of \cite{Sidman:AIM02}, which says that the maximal degree of a generator of $\I_{\Phi}$ does not exceed $n$.
\end{proof}	
\begin{example} \label{example:NonTransversal}
	Let $\Phi=\cS_1 \cup \cS_2$ be the union of two planes of $\Re^3$ with normal vectors $\b_1,\b_2$, and let $\X=\left\{\x_1,\x_2,\x_3,\x_4\right\}$ be four points of $\Phi$, such that, $\x_1,\x_2 \in \cS_1 \setminus \cS_2$ and $\x_3,\x_4 \in \cS_2 \setminus\cS_1$. Let $\H_{13}$ and $\H_{24}$ be the planes spanned by $\x_1,\x_3$ and $\x_2,\x_4$ respectively, and let $\b_{13}, \b_{24}$ be the normals to these planes. Then the polynomial $q(x) = (\b_{13}^\transpose x)(\b_{24}^\transpose x)$ certainly vanishes on $\X$. But $q$ does not vanish on $\Phi$, because the only (up to a scalar) homogeneous polynomial of degree $2$ that vanishes on $\Phi$ is $p(x)=(\b_{1}^\transpose x)(\b_{2}^\transpose x)$. Hence $\X$ is not in general position in $\Phi$. The geometric reasoning  is that two points per plane are not enough to uniquely define the union of the two planes; instead a third point in one of the planes is required.
\end{example}

\noindent In terms of a finite set of points $\X$, Theorem \ref{thm:ASC} becomes:
\begin{thm} \label{thm:ASCfinite}
	Let $\X$ be a finite set of points sampled from a union
	$\Phi$ of $n$ linear subspaces of
	$\Re^D$. Let $p_1,\dots,p_s$ be a basis for $\I_{\X,n}$, the vector space of homogeneous polynomials of degree $n$ that vanish on $\X$. Let $\x_i$ be a point in $\X_i:= \X \cap \cS_i$ such that $\x_i \not\in \bigcup_{i' \neq i} \cS_{i'}$. If $\X$ is in general position in $\Phi$ (Definition \ref{dfn:DensePointsLinear}), and $\Phi$ is transversal (Definition \ref{dfn:transversal}), then
	$\cS_i = \Span( \nabla p_1|_{\x_i}, \dots,  \nabla p_s|_{\x_i} )^{\perp}$.
\end{thm}

\section{Problem Statement and Contributions} \label{section:ProblemOverview}
In this section we begin by defining the problem of clustering unions of affine subspaces in Section \ref{subsection:AffineSubspaceClustering}. In Section \ref{subsection:traditional} we analyze the traditional algebraic approach for handling affine subspaces and point out that its correctness is far from obvious. Finally, in Section \ref{subsection:contributions} we state the main findings of this paper. 

\subsection{Affine Subspace Clustering Problem} \label{subsection:AffineSubspaceClustering}
Let $\X=\left\{\x_1,\dots,\x_N\right\}$ be a finite set of points living in a union $\Psi=\bigcup_{i=1}^n\A_i$ of $n$ affine subspaces of $\Re^D$. Each affine subspace $\A_i$ is the translation by some vector $\bmu_i \in \Re^D$ of a $d_i$-dimensional linear subspace $\cS_i$, i.e., $\A_i = \cS_i + \bmu_i$. The affine subspace clustering problem involves clustering the points $\X$ according to their subspace membership, and finding a parametrization of each affine subspace $\A_i$ by finding a translation vector $\bmu_i$ and a basis for its linear part $\cS_i$, for all $ i=1,\dots,n$. Note that there is an inherent ambiguity in determining the translation vectors $\bmu_i$, since if $\A_i = \cS_i + \bmu_i$, then $\A_i = \cS_i + (\s_i+\bmu_i)$ for any vector $\s_i \in \cS_i$. Consequently, the best we can hope for is to determine the unique component of $\bmu_i$ in the orthogonal complement $\cS_i^\perp$ of $\cS_i$.

\subsection{Traditional Algebraic Approach} \label{subsection:traditional}
Since the inception of ASC, the standard algebraic approach to cluster points living in a union of affine subspaces has been to embed the points into $\Re^{D+1}$ and subsequently apply ASC \cite{Vidal:PhD03}. 
The precise embedding $\phi_0: \Re^D \hookrightarrow \Re^{D+1}$ is given by
\begin{align}
 \boldsymbol{\alpha}=(\alpha_1,\dots,\alpha_D) \stackrel{\phi_0}{\longmapsto} \tilde{\boldsymbol{\alpha}}=
(1,\alpha_1,\dots,\alpha_D). \label{eq:phi0}
\end{align}
To understand the effect of this embedding and why it is meaningful to apply ASC to the embedded points, let $\A = \cS + \bmu$ be a $d$-dimensional affine subspace of $\Re^D$, with $\u_1,\dots,\u_d$ being a basis for its linear part $\cS$. As noted in Section \ref{subsection:AffineSubspaceClustering}, we can also assume that $\bmu \in \cS^\perp$.
For $\x \in \A$, there exists $\y \in \Re^d$ such that 
\begin{align}
\x &= \bU  \y + \bmu, \, \, \, 
\bU := \left[\u_1,\dots,\u_d\right] \in \Re^{D \times d}.
\end{align}
Then the embedded point
$\xtb:=\phi_0(\x)$ can be written as 
\begin{align}
\xtb =\begin{bmatrix}1 \\ \x \end{bmatrix}= \btU \left[ \begin{array}{c} 1 \\ \y \end{array} \right], \, \, \, 
\btU := \left[\begin{array}{cccc}  1 & 0 & \cdots & 0 \\ \bmu & \u_1 & \cdots & \u_d \end{array}\right].
\label{eq:EmbeddedPoint}
\end{align} 
Equation \eqref{eq:EmbeddedPoint} clearly indicates that the embedded point $\xtb$ lies in the linear $(d+1)$-dimensional subspace $\tilde{\cS}:=\Span(\btU)$ of $\Re^{D+1}$ and the same is true for the entire affine subspace $\A$. From \eqref{eq:EmbeddedPoint} one sees immediately that $(\u_1,\dots,\u_d,\bmu)$ can be used to construct a basis of $\tilde{\cS}$. The converse is also true: given any basis of $\tilde{\cS}$ one can recover a basis for the linear part $\cS$ and the translation vector $\bmu$ of $\A$.
Hence, the embedding $\phi_0$
takes a union of affine subspaces $\Psi = \bigcup_{i=1}^n\A_i$ into a union of linear subspaces $\tilde{\Phi}=\bigcup_{i=1}^n\tilde\cS_i$ of $\Re^{D+1}$, in a way that there is a $1-1$ correspondence between the parameters of $\A_i$ (a basis for the linear part and the translation vector) and the parameters of $\tilde{\cS}_i$ (a basis) for every $i \in [n]$.

To the best of our knowledge, the correspondence between $\A_i$ and $\tilde{\cS}_i$ has been the sole theoretical justification so far in the subspace clustering literature for the traditional Affine ASC (AASC) approach for dealing with affine subspaces, which consists of
\begin{enumerate}
	\item applying the embedding $\phi_0$ to points $\X$ in $\Psi$,
	\item computing a basis $p_1,\dots,p_s$ for the vector space  $\I_{\tilde{\X},n}$ of homogeneous polynomials of degree $n$ that vanish on the embedded points $\tilde{\X}:=\phi_0(\X)$, 
	\item for $\tilde{\x}_i \in \tilde{\X}\cap \tilde{\cS}_i \setminus \bigcup_{i \neq i'} \tilde{\cS}_i$, estimating $\tilde{\cS}_i$ via the formula 
	\begin{align} \tilde{\cS}_i = \Span(\nabla p_1|_{\tilde{\x}_i}, \dots,  \nabla p_s|_{\tilde{\x}_i})^{\perp},\label{eq:estimationSi} \end{align} 
	\item and extracting the translation vector of $\A_i$ and a basis for its linear part from a basis of $\tilde{\cS}_i$.
\end{enumerate} According to Theorem \ref{thm:ASCfinite}, the above process will succeed, if i) the embedded points $\tilde{\X}$ are in general position in $\tilde{\Phi}$ (in the sense of Definition \ref{dfn:DensePointsLinear}), and ii) the union of linear subspaces $\tilde{\Phi}$ is transversal. Note that these conditions need not be satisfied a priori because of the particular structure of both the embedded data in \eqref{eq:embeddingIntro} and the basis in \eqref{eq:EmbeddedPoint}. This gives rise to the following reasonable questions:
\begin{que}\label{que:General}
Under what conditions on $\X$ and $\Psi$, will $\tilde{\X}$ be in general position in $\tilde{\Phi}$?
\end{que}
\begin{que}\label{que:Transversal}
	Under what conditions on $\Psi$ will $\tilde{\Phi}$ be transversal?	 	
\end{que} 

\subsection{Contributions} \label{subsection:contributions} 
The main contribution of this paper is to answer
Questions \ref{que:General}-\ref{que:Transversal}. 

Regarding Question \ref{que:General}, one may be tempted to conjecture that $\tilde{\X}$ is in general position in $\tilde{\Phi}$, if the components of the points $\X$ along the union $\Phi:=\bigcup_{i=1}^n \cS_i$ of the linear parts of the affine subspaces are in general position inside $\Phi$. However, this conjecture is not true, as illustrated by the next example.

\begin{example}
	Suppose that $\Psi = \A_1 \cup \A_2$ is a union of two affine planes $\A_i = \cS_i + \bmu_i$ of $\Re^3$. Then $\Phi = \cS_1 \cup \cS_2$ is a union of $2$ planes in $\Re^3$ and 
	as argued in Example \ref{example:NonTransversal}, we can find $5$ points in general position in $\Phi$. However, $\tilde{\Phi}=\tilde{\cS}_1 \cup \tilde{\cS}_2$ is a union of $2$ hyperplanes in $\Re^4$ and any subset of $\tilde{\Phi}$ in general position must consist of at least $\mathcal{M}_2(4)-1 = {2 + 3 \choose 2}-1 = 9$ points.\footnote{Otherwise one can fit a polynomial of degree $2$ to the points, which does not vanish on $\tilde{\Phi}$.}
\end{example}

\noindent To state the precise necessary and sufficient condition for $\tilde{\X}$ to be in general position in $\tilde{\Phi}$,
we first show that $\Psi$ is the zero-set of non-homogeneous polynomials of degree $n$.

\begin{prp}\label{prp:UnionAffineZeroSet}
	Let $\Psi = \bigcup_{i=1}^n \A_i$ be a union of affine subspaces of $\Re^D$, where each affine subspace $\A_i$ is the translation of a linear subspace $\cS_i$ of codimension $c_i$ by a translation vector $\bmu_i$. For each $\A_i=\cS_i+\bmu_i$, let $\b_{i1},\dots,\b_{ic_i}$ be a basis for $\cS_i^\perp$. Then $\Psi$ is the zero set of all degree-$n$ polynomials of the form
	\begin{align}
	\!\!\!
	\prod_{i=1}^n \big(\b_{i j_i}^\transpose x \!-\! \b_{ij_i}^\transpose \bmu_i\big) : (j_1,\dots,j_n) \in [c_1]\times \cdots \times [c_n]. \!\! \label{eq:GeneratorsPsi}
	\end{align}
\end{prp} 

\noindent Thanks to Proposition \ref{prp:UnionAffineZeroSet} we can define points $\X$ to be in general position in $\Psi$, in analogy to Definition \ref{dfn:DensePointsLinear}.

\begin{dfn}\label{dfn:DensePointsAffine}
	Let $\Psi$ be a union of $n$ affine subspaces of $\Re^D$ and $\X$ a finite subset of $\Psi$. We will say that $\X$ is in general position in $\Psi$, if $\Psi$ can be recovered as the zero set of all polynomials of degree $n$ that vanish on $\X$. Equivalently, a polynomial of degree $n$ vanishes on $\Psi$ if and only if it vanishes on $\X$.
\end{dfn} 
\noindent We are now ready to answer our Question \ref{que:General}.

\begin{thm}\label{thm:MainTheoremDense}
Let $\X$ be a finite subset of a union of $n$ affine subspaces $\Psi=\bigcup_{i=1}^n \A_i$ of $\Re^D$, where $\A_i = \cS_i+\bmu_i$, with $\cS_i$ a linear subspace of $\Re^D$ of codimension $0< c_i < D$. Let $\tilde{\Phi} = \bigcup_{i=1}^n \tilde{\cS_i}$ be the union of $n$ linear subspaces of $\Re^{D+1}$ induced by the embedding $\phi_0:\Re^D \hookrightarrow \Re^{D+1}$ in \eqref{eq:phi0}. Denote by $\tilde{\X} \subset \tilde{\Phi}$ the image of $\X$ under $\phi_0$. Then $\tilde{\X}$ is in general position in $\tilde{\Phi}$ if and only if $\X$ is in general position in $\Psi$.
\end{thm}

\noindent Our second Theorem answers Question \ref{que:Transversal}.

\begin{thm}\label{thm:MainTransversal}
	Let $\Psi = \bigcup_{i=1}^n \A_i$ be a union of $n$ affine subspaces of $\Re^D$, with $\A_i = \cS_i + \bmu_i$ and  $\bmu_i = \B_i \a_i$, where $\B_i \in \Re^{D \times c_i}$ is a basis for $\cS_i^\perp$ with $c_i=\codim \cS_i$. If $\Phi = \bigcup_{i=1}^n \cS_i$ is transversal and $\a_1,\dots,\a_n$ do not lie in the zero-measure set of a proper algebraic variety\footnote{The precise description of this algebraic variety is given in the proof of the Theorem in Section \ref{subsection:ProofTransversal}.} of $\Re^{c_1} \times \cdots \times \Re^{c_n}$, then $\tilde{\Phi}$ is transversal.
\end{thm} 

\noindent One may wonder if some of the $\bmu_i$ can be zero and $\tilde{\Phi}$ still be transversal. This depends on the $c_i$ as the next example shows.

\begin{example}
Let $\A_1 = \Span(\b_{11},\b_{12})^\perp + \bmu_1$ be an affine line and $\A_2 = \Span(\b_{2})^\perp + \bmu_2$ an affine plane of $\Re^3$. Suppose that $\Phi=\Span(\b_{11},\b_{12})^\perp \cup \Span(\b_{2})^\perp$ is transversal. Then $\tilde{\Phi}=\tilde{\cS}_1 \cup \tilde{\cS}_2$ is transversal if and only if the matrix
\begin{align}
\tilde{\B} := \begin{bmatrix}
-\b_{11}^\transpose \bmu_1 & -\b_{12}^\transpose \bmu_1 & -\b_{2}^\transpose \bmu_2 \\
\b_{11} & \b_{12} & \b_{2} 
\end{bmatrix} \in \Re^{4 \times 3}
\end{align} has rank $3$. But $\rank\left(\tilde{\B}\right)=3$, irrespectively of what the $\bmu_i$ are, simply because the $3 \times 3$ matrix $\B := [\b_{11} \, \,  \b_{12} \, \, \b_{2}]$ is full rank (by the transversality assumption on $\Phi$). Now let us replace the affine plane $\A_2$ with a second affine line $\A_2 = \Span(\b_{21},\b_{22})^\perp + \bmu_2$. Then $\tilde{\Phi}$ is transversal if and only if 
\small
\begin{align}
\tilde{\B}: = \begin{bmatrix}
-\b_{11}^\transpose \bmu_1 & -\b_{12}^\transpose \bmu_1 & -\b_{21}^\transpose \bmu_2 & -\b_{22}^\transpose \bmu_2 \\
\b_{11} & \b_{12} & \b_{21} & \b_{22}
\end{bmatrix} \in \Re^{4 \times 4}
\end{align} \normalsize has rank $4$, which is impossible if both $\bmu_1,\bmu_2$ are zero.
\end{example}

\noindent As a corollary of Theorems \ref{thm:ASCfinite}, \ref{thm:MainTheoremDense} and \ref{thm:MainTransversal}, we get the correctness Theorem of ASC for the case of affine subspaces.

\begin{thm} \label{thm:MainCorrectness}
	Let $\Psi = \bigcup_{i=1}^n \A_i$ be a union of affine subspaces of $\Re^D$, with $\A_i = \cS_i + \bmu_i$ and  $\bmu_i = \B_i \a_i$, where $\B_i \in \Re^{D \times c_i}$ is a basis for $\cS_i^\perp$ with $c_i=\codim \cS_i$. Let $\tilde{\Phi} = \bigcup_{i=1}^n \tilde{\cS_i}$ be the union of $n$ linear subspaces of $\Re^{D+1}$ induced by the embedding $\phi_0:\Re^D \hookrightarrow \Re^{D+1}$ of \eqref{eq:phi0}. Let $\X$ be a finite subset of $\Psi$ and denote by $\tilde{\X} \subset \tilde{\Phi}$ the image of $\X$ under $\phi_0$. Let $p_1,\dots,p_s$ be a basis for $\I_{\tilde{\X},n}$, the vector space of homogeneous polynomials of degree $n$ that vanish on $\tilde{\X}$. Let $\x \in \X \cap \A_1 \setminus \bigcup_{i>1} \A_i$, and denote $\tilde{\x} = \phi_0(\x)$.
	 Define
	\begin{align}
	\bt_k &:= \nabla p_k|_{\tilde{\x}_1} \in \Re^{D+1}, \, \, \, k=1,\dots,s,
	\end{align} and without loss of generality, let $\bt_1,\dots,\bt_{\ell}$ be 
	a maximal linearly independent subset of $\bt_1,\dots,\bt_s$.
	Define further $(\gamma_{k}, \b_{k}) \in \Re \times \Re^D$ and 
	$(\g_1,\B_1) \in \Re^\ell \times \Re^{D \times \ell}$ as
	\begin{align}
	\bt_k &=: \left[ \begin{array}{c} \gamma_k \\ \b_{k} \end{array} \right], \, k=1,\dots, \ell\\
	\g_1 & := \left[\gamma_1,\dots,\gamma_{\ell} \right]^\transpose, \, \, \, 
	\B_1  := \left[\b_1, \dots, \b_{\ell} \right].
	\end{align} 
	If $\X$ is in general position in $\Psi$, $\Phi = \bigcup_{i=1}^n \cS_i$ is transversal, and $\a_1,\dots,\a_n$ do not lie in the zero-measure set of a proper algebraic variety of $\Re^{c_1} \times \cdots \times \Re^{c_n}$, then
	\begin{align}
	\A_1 = \Span(\B_1)^\perp - \B_1 \left( \B_1^\transpose \B_1 \right)^{-1} \g_1. \label{eq:A1representation}
	\end{align}
\end{thm}

\begin{remark} The acute reader may notice that we still need to answer the question of whether $\Psi$ admits a finite subset $\X$ in general position, to begin with. This answer is affirmative: If $\Psi$ satisfies the hypothesis of Theorem \ref{thm:MainTransversal}, then $\tilde{\Phi}$ will be transversal, and so by Proposition \ref{prp:AffineUnionIdealHomogenization} $\I_{\Psi}$ is generated in degree $\le n$, in which case the existence of $\X$ follows from Theorem 2.9 in \cite{Ma:SIAM08}.
\end{remark}

The rest of the paper is organized as follows: in Section \ref{section:AlgebraicGeometryAffine} we establish the fundamental algebraic-geometric properties of a union of affine subspaces. Then using these tools, we prove in Section \ref{section:Proof} Theorems \ref{thm:MainTheoremDense} and \ref{thm:MainTransversal}. The proof of Theorem \ref{thm:MainCorrectness} is straightforward is thus omitted.

\section{Algebraic Geometry of Unions of Affine Subspaces} \label{section:AlgebraicGeometryAffine}
In Section \ref{subsection:AffineSubspaceVariety} we describe the basic algebraic geometry of affine subspaces and unions thereof, in analogy to the case of linear subspaces. 
In particular, we show that a single affine subspace is the zero-set of polynomial equations of degree $1$, and a union $\Psi$ of affine subspaces is the zero-set of polynomial equations of degree $n$. 
In Section \ref{subsection:ProjectiveClosure} we study more closely the embedding $ \A \stackrel{\phi_0}{\longrightarrow} \tilde{\cS}$ of an affine subspace $\A \subset \Re^D$ into its associated linear subspace $\tilde{\cS} \subset \Re^{D+1}$ (see Section \ref{subsection:traditional}), which will lead to a deeper understanding of the embedding $ \Psi \stackrel{\phi_0}{\longrightarrow} \tilde{\Phi}$ of a union of affine subspaces $\Psi \subset \Re^D$ into its associated union of linear subspaces $\tilde{\Phi} \subset \Re^{D+1}$. As we will see, $\phi_0(\Psi)$ is \emph{dense} in $\tilde{\Phi}$ in a very precise sense, and the algebraic manifestation of this relation (Proposition \ref{prp:AffineUnionIdealHomogenization}) will be used later in 
 Section \ref{subsection:ProofDense}, to prove our Theorem \ref{thm:MainTheoremDense}.

\subsection{Affine Subspaces as Affine Varieties} \label{subsection:AffineSubspaceVariety}

Let $\A = \cS + \bmu$ be an affine subspace of $\Re^D$ and let
$\b_1,\dots,\b_c$ be a basis for the orthogonal complement $\cS^\perp$ of $\cS$. 
The first important observation is that a vector $\x$ belongs to $\cS$ if and only if
$ \x \perp \b_k, \, \, \, \forall k =1,\dots,c$. In the language of algebraic geometry
this is the same as saying that $\cS$ is the zero set of $c$ linear polynomials:
\begin{align}
\cS = \Z\left(\b_1^\transpose x, \dots, \b_c^\transpose x\right), \, \, \, x := [x_1,\dots,x_D]^\transpose. \label{eq:Svariety}
\end{align} 
\begin{dfn} \label{dfn:VanishingIdeal}
	Let $\Y$ be a subset of $\Re^{D}$. The set $\I_{\Y}$ of polynomials $p(x_1,\dots,x_{D})$ that vanish on $\Y$, i.e., $p(y_1,\dots,y_{D})=0, \, \forall [y_1,\dots,y_{D}]^\transpose \in \Y$, is called the vanishing ideal of $\Y$.
\end{dfn} \noindent One may wonder if the linear polynomials $\b_i^\transpose x, \, i=1,\dots,c$, form some sort of \emph{basis} for the vanishing ideal $\I_{\cS}$ of $\cS$. In fact this is true (see the appendix in \cite{Tsakiris:SIAM17} for a proof) and can be formalized by saying that these linear polynomials are \emph{generators} of $\I_{\cS}$ over the polynomial ring $\R=\Re[x_1,\dots,x_D]$. This means that every polynomial that belongs to $\I_\cS$
can be written as a linear combination of $\b_1^\transpose x, \dots, \b_c^\transpose x$
with polynomial coefficients, i.e., 
\begin{align}
p(x) = p_1(x) (\b_1^\transpose x) + \cdots + p_c(x) (\b_c^\transpose x) \label{eq:IdealGenerated}
\end{align} where $p_1,\dots,p_c$ are some polynomials in $\R$.
More compactly 
\begin{align}
\I_{\cS} = (\b_1^\transpose x, \dots, \b_c^\transpose x), \label{eq:LinearSubspaceVanishingIdeal}
\end{align}
which reads as $\I_{\cS}$ is \emph{the ideal generated by} the polynomials $\b_1^\transpose x, \dots, \b_c^\transpose x$ as in \eqref{eq:IdealGenerated}. The following important fact\footnote{For a proof see Appendix C in \cite{Tsakiris:SIAM17}.} will be used in Section \ref{subsection:ProofDense} to prove our Theorem \ref{thm:MainTheoremDense}.
\begin{prp}\label{prp:LinearSubspaceIdeal}
	The vanishing ideal $\I_{\cS}$ of a linear subspace $\cS$ is always a prime ideal, i.e., if $p,q$ are polynomials such that 
	$p q \in \I_{\cS}$, then either $p \in \I_{\cS}$ or $q \in \I_{\cS}$.
\end{prp}
\noindent Moving on, the second important observation is that $\x \in \A$ if and only if $\x - \bmu \in \cS$.
Equivalently, 
\begin{align}
\x \in \A \, \, \, \Leftrightarrow \, \, \, \b_k \perp \x - \bmu, \, \, \, \forall k=1,\dots,c
\end{align} or in algebraic geometric terms
\begin{align}
\A = \Z \left(\b_1^\transpose x - \b_1^\transpose \bmu,\dots,\b_c^\transpose x - \b_c^\transpose \bmu\right). \label{eq:DIaffine}
\end{align} In other words, the affine subspace $\A$ is an algebraic variety of $\Re^D$. In fact, we say that $\A$ is an \emph{affine variety}, since it is defined by non-homogeneous polynomials.
To describe the \emph{vanishing ideal} $\I_{\A}$ of $\A$, note that 
a polynomial $p(x)$ vanishes on $\A$ if and only if $p(x+\bmu)$ vanishes on
$\cS$. This, together with \eqref{eq:LinearSubspaceVanishingIdeal}, give
\begin{align}
\I_{\A} =  \left(\b_1^\transpose x - \b_1^\transpose \bmu,\dots,\b_c^\transpose x - \b_c^\transpose \bmu\right).
\end{align} 

Next, we consider a union $\Psi =\bigcup_{i=1}^n \A_i$ of affine subspaces
$\A_i = \cS_i + \bmu_i, \, i\in [n]$, of $\Re^D$. We will prove Proposition \ref{prp:UnionAffineZeroSet}, which describes $\Psi$ as the zero-set of non-homogeneous polynomials of degree $n$, showing that $\Psi$ is an affine variety of $\Re^D$.

\begin{proof}(Proposition \ref{prp:UnionAffineZeroSet})
	Denote the set of all polynomials of the form \eqref{eq:GeneratorsPsi} by $\mathcal{P}$. First, we show that $\Psi \subset \Z(\mathcal{P})$. Take $\x \in \Psi$; we will show that $\x \in \Z(\mathcal{P})$. Since $\Psi = \A_1\cup\cdots \cup \A_n$, $\x$ belongs to at least one of the affine subspaces, say $\x \in \A_i$, for some $i$. For every polynomial $p$ of $\mathcal{P}$, there is a linear factor $\b_{i j_i}^\transpose x- \b_{ij_i}^\transpose \bmu_i$ of $p$ that vanishes on $\A_i$ and thus on $\x$. Hence $p$ itself will vanish on $\x$. Since $p$ was an arbitrary element of $\mathcal{P}$, this shows that every polynomial of $\mathcal{P}$ vanishes on $\x$, i.e., $\x \in \Z(\mathcal{P})$.
	
	Next, we show that $\Z(\mathcal{P}) \subset \Psi$. 
	Let $\x \in \Z(\mathcal{P})$; we will show that $\x \in \Psi$. If 
	$\x$ is a root of all polynomials $p_{1j}(x) = \b_{1 j}^\transpose x - \b_{1j}^\transpose \bmu_1$,
	then $\x \in \A_1$ and we are done. Otherwise, one of these linear polynomials does not vanish on $\x$, say $p_{1 1}(\x) \neq 0$.
	Now suppose that $\x \not\in \Psi$. By the above argument, for every affine subspace $\A_i$ there must exist some linear polynomial 
	$\b_{i1}^\transpose x - \b_{i1}^\transpose \bmu_i$, which does not vanish on $\x$. As consequence, the polynomial	
	\begin{align}
	p(x)=\prod_{i=1}^n \big(\b_{i 1}^\transpose x - \b_{i 1}^\transpose \bmu_i\big) 	
	\end{align} does not vanish on $\x$, i.e., $p(\x) \neq 0$. But because of the definition of $\mathcal{P}$, we must have that $p \in \mathcal{P}$. Since $\x$ was selected to be an element of
	$\Z(\mathcal{P})$, we must have that $p(\x)=0$, which is a contradiction, as we just saw that $p(\x)\neq 0$. Consequently, the hypothesis that
	$\x \not\in \Psi$, must be false, i.e., $\Z(\mathcal{P}) \subset \Psi$.
\end{proof} 

The reader may wonder what the vanishing ideal $\I_{\Psi}$ of $\Psi$ is and what is its relation to the linear polynomials whose products generate $\Psi$, as in Proposition \ref{prp:UnionAffineZeroSet}. In fact, this question is still partially open even in the simpler case of a union of linear subspaces \cite{Conca:CM03,Sidman:AIM02,Derksen:JPAA07}. As it turns out, $\I_{\Psi}$ is intimately related to $\I_{\tilde{\Phi}}$, where $\tilde{\Phi}=\bigcup_{i=1}^n\tilde{\cS}_i$ is the union of linear subspaces associated to $\Psi$ under the embedding $\phi_0$ of \eqref{eq:phi0}. It is precisely this relation that will enable us to prove Theorem \ref{thm:MainTheoremDense}, and to elucidate it we need the notion of \emph{projective closure} that we introduce next.\footnote{Of course, the notion of \emph{projective} closure is a well-known concept in algebraic geometry; here we introduce it in a self-contained fashion in the context of unions of affine subspaces, dispensing with unnecessary abstractions.}

\subsection{The Projective Closure of Affine Subspaces} \label{subsection:ProjectiveClosure}
Let $\phi_0(\A)$ be the image of $\A = \cS + \bmu$ under the embedding $\phi_0: \Re^D \hookrightarrow \Re^{D+1}$ in \eqref{eq:phi0}. Let $\tilde{\cS}$ be the $(d+1)$-dimensional linear subspace 
of $\Re^{D+1}$ spanned by the columns of $\btU$ (see \eqref{eq:EmbeddedPoint}).
A basis for the orthogonal complement of $\tilde{\cS}$ in $\Re^{D+1}$ is 
\begin{align}
\bt_1:=\left[ \begin{array}{c} -\b_1^\transpose \bmu \\ \b_1 \end{array}\right], \dots,
\bt_c:=\left[\begin{array}{c} -\b_c^\transpose \bmu \\ \b_c \end{array}\right],
\end{align} since $\codim (\tilde{\cS}_i) = \codim (\cS)$, and the $\tilde{\b}_i$ are linearly independent because the $\b_i$ are.
In algebraic geometric terms
\begin{align}
\label{eq:PCsimilarity}
\begin{split}
\tilde{\cS} &= \Z \left(\b_1^\transpose x - (\b_1^\transpose \bmu) x_0,\dots,\b_c^\transpose x - (\b_c^\transpose\bmu) x_0\right)\\
&=\Z \left(\bt_1^\transpose \xt,\dots,\bt_c^\transpose \xt \right), \, \xt:=[x_0,x_1,\dots,x_D]^\transpose.
\end{split}
\end{align} 
By inspecting equations \eqref{eq:DIaffine} and \eqref{eq:PCsimilarity}, we see that every point of $\phi_0(\A)$ satisfies the equations \eqref{eq:PCsimilarity} of $\tilde{\cS}$. Since these equations are homogeneous, it will in fact be true that for any point $\tilde{\x} \in \phi_0(\A)$ the entire line of $\Re^{D+1}$ spanned by $\tilde{\x}$ will still lie in $\tilde{\cS}$. Hence, we may as well think of the embedding $\phi_0$ as mapping a point $\x \in \Re^D$ to a line of $\Re^{D+1}$. To formalize this concept, we need the notion of \emph{projective space} \cite{Cox:2007,Hartshorne-1977}:

\begin{dfn}\label{dfn:ProjectiveSpace}
The real projective space $\Pr^D$ is defined to be the set of all lines through the origin in $\Re^{D+1}$. Each non-zero vector  $\boldsymbol{\alpha}$ of $\Re^{D+1}$ defines an element $[\boldsymbol{\alpha}]$ of $\Pr^D$, and two elements $[\boldsymbol{\alpha}], [\boldsymbol{\beta}]$ of $\Pr^D$ are equal in $\Pr^D$, if and only if there exists a nonzero $\lambda \in \Re$ such that we have an equality $\boldsymbol{\alpha}= \lambda \boldsymbol{\beta}$ of vectors in $\Re^{D+1}$. For each point $[\boldsymbol{\alpha}] \in \Pr^D$, we call the point $\boldsymbol{\alpha} \in \Re^{D+1}$ a representative of $[\boldsymbol{\alpha}]$.
\end{dfn} 

Now we can define a new embedding $\hat{\phi}_0 : \Re^D \rightarrow \Pr^D$, that behaves exactly as $\phi_0$ in \eqref{eq:phi0}, except that it now takes points of $\Re^D$ to lines of $\Re^{D+1}$, or more precisely, to elements of $\Pr^D$:
\begin{align}
(\alpha_1,\alpha_2,\dots,\alpha_D) \stackrel{\hat{\phi}_0}{\longmapsto} [(1,\alpha_1,\alpha_2,\cdots,\alpha_D)].
\label{eq:projective_phi0}
\end{align} 
A point $\x$ of $\A$ is mapped by $\hat{\phi}_0$ to a line inside $\tilde{\cS}$, or more specifically, to the point $[\tilde{\x}]$ of $\Pr^D$, whose representative $\tilde{\x}$ satisfies the equations \eqref{eq:PCsimilarity} of $\tilde{\cS}$. The set of all lines of $\Re^{D+1}$ that live in $\tilde{\cS}$, viewed as elements of $\Pr^D$, is denoted by $[\tilde{\cS}]$, i.e., 
\begin{align}
[\tilde{\cS}] = \left\{[\boldsymbol{\alpha}] \in \Pr^D: \, \boldsymbol{\alpha} \in \tilde{\cS} \right\}. 
\end{align} \noindent The representative $\boldsymbol{\alpha}$ of every element $[\boldsymbol{\alpha}] \in [\tilde{\cS}]$ satisfies by definition the equations \eqref{eq:PCsimilarity} of $\tilde{\cS}$, and so $[\tilde{\cS}]$ has naturally the structure of an algebraic variety of $\Pr^D$, which is called a projective variety. We emphasize that even though the varieties $\tilde{\cS}$ and $[\tilde{\cS}]$ live in different spaces, $\Re^{D+1}$ and $\Pr^D$ respectively, they are defined by the same equations. In fact, every algebraic variety $\Y$ of $\Re^{D+1}$ that is the unions of lines, which is true if and only if $\Y$ is defined by homogeneous equations, gives rise to a projective variety $[\Y]$ of $\Pr^D$ defined by the same equations.

\begin{example}
	 Recall from Section \ref{subsection:UnionsLinearVarieties} that a union $\tilde{\Phi}$ of linear subspaces is defined as the zero-set of homogeneous polynomials. Then $\tilde{\Phi}$ gives rise to a projective variety $[\tilde{\Phi}]$ of $\Pr^D$ defined by the same equations as $\tilde{\Phi}$, which can be thought of as the set of lines through the origin in $\Re^{D+1}$ that live in $\tilde{\Phi}$.
\end{example}

\noindent Returning to our embedding $\hat{\phi}_0$, to describe the precise connection between $\hat{\phi}_0(\A)$ and $[\tilde{\cS}]$ we need to resort
to the kind of topology that is most suitable for the study
of algebraic varieties \cite{Cox:2007,Hartshorne-1977}:
\begin{dfn}[Zariski Topology] \label{dfn:Zariski}
The real vector space $\Re^D$ and the projective space $\Pr^D$ can be made into topological spaces,
by defining the closed sets of their associated topology to be all the algebraic 
varieties in $\Re^D$ and $\Pr^D$ respectively.  
\end{dfn} \noindent We are finally ready to state without proof the formal algebraic geometric relation between $\hat{\phi}_0(\A)$ and $\tilde{\cS}$:

\begin{prp} \label{prp:SingleSubspaceProjectiveClosure}
In the Zariski topology, the set $\hat{\phi}_0(\A)$ is open and dense in $[\tilde{\cS}]$, and so $[\tilde{\cS}]$ is the closure\footnote{It can further be shown that $[\tilde{\cS}] = \hat{\phi}_0(\A) \cup [\cS]$: intuitively, the set that we need to add to $\hat{\phi}_0(\A)$ to get a closed set is the \emph{slope} $[\cS]$ of $\A$.} of $\hat{\phi}_0(\A)$ in $\Pr^D$. 
\end{prp}

The projective variety $[\tilde{\cS}]$ is called the \emph{projective closure} of $\A$: it is the smallest 
projective variety that contains $\hat{\phi}_0(\A)$. We now characterize the projective closure of a union of affine subspaces. 

\begin{prp}
Let $\Psi = \bigcup_{i=1}^n\A_i$ be a union of affine subspaces of $\Re^D$. Then the
projective closure of $\Psi$ in $\Pr^D$, i.e., the smallest projective variety that contains $\hat{\phi}_0(\Psi)$, is
\begin{align}
\bigcup_{i=1}^n[\tilde{\cS}_i] = \left[\bigcup_{i=1}^n\tilde{\cS}_i\right] = \left[\tilde{\Phi}\right],
\end{align} where $\tilde{\cS}_i$ is the linear subspace of $\Re^{D+1}$ corresponding to $\A_i$ under the embedding $\phi_0$ of \eqref{eq:phi0}.
\end{prp}

\noindent The geometric fact that $[\tilde{\Phi}] \subset \Pr^D$ is the smallest projective variety of $\Pr^D$ that contains $\hat{\phi}_0(\Psi)$, manifests itself algebraically in $\I_{\Psi}$ being uniquely defined by $\I_{\tilde{\Phi}}$ and vice versa, in a very precise fashion. To describe this relation, we need a definition.
\begin{dfn}[Homogenization - Dehomogenization]
	Let $p \in \R=\Re[x_1,\dots,x_D]$ be a polynomial of degree $n$. The homogenization of $p$ is the homogeneous polynomial  
	\begin{align}
	p^{(h)} = x_0^n p\left(\frac{x_1}{x_0},\frac{x_2}{x_0},\dots,\frac{x_D}{x_0}\right) 
	\label{eq:homogenization}
	\end{align} of $\Rt=\Re[x_0,x_1,\dots,x_D]$ of degree $n$.
	Conversely, if $P \in \Rt$ is homogeneous of degree $n$,
	its dehomogenization is 
	$P_{(d)} = P(1,x_1,\dots,x_D)$, which is a polynomial of $\R$ of degree $\le n$.
\end{dfn}

\begin{example}\label{example:DehomogenizationDegreeDrop}
	Let $P=x_0^2x_1+x_0x_2^2+x_1x_2x_3$ be a homogeneous polynomial of degree $3$. Its dehomogenization is the degree-$3$ polynomial $P_{(d)}=x_1+x_2^2+x_1x_2x_3$, and the homogenization of $P_{(d)}$ is 	
		$\left(P_{(d)}\right)^{(h)} = x_0^3\left(\frac{x_1}{x_0}+\frac{x_2^2}{x_0^2}+\frac{x_1x_2x_3}{x_0^3}\right) = P$.
\end{example} \noindent The next result from algebraic geometry is crucial for our purpose.
\begin{thm}[Chapter $8$ in \cite{Cox:2007}]\label{thm:ProjectiveClosureHomogenization}
	Let $\Y$ be an affine variety of $\Re^D$ and let $\bar{\Y}$ be its projective closure in $\Pr^D$ with respect to the embedding $\hat{\phi}_0$ of \eqref{eq:projective_phi0}. Let $\I_{\Y}, \I_{\bar{\Y}}$ be the vanishing ideals of $\Y, \bar{\Y}$ respectively. Then $\I_{\bar{\Y}}=\I_{\Y}^{(h)}$, i.e., every element of $\I_{\bar{\Y}}$ arises as a homogenization of some element of $\I_{\Y}$, and every element of $\I_{\Y}$ arises as the dehomogenization of some element of $\I_{\bar{\Y}}$.
\end{thm} 

\noindent We have already seen that $\tilde{\Phi}$ and $[\tilde{\Phi}]$ are given as algebraic varieties by identical equations. It is also not hard to see that the vanishing ideals of these varieties are identical as well.
\begin{lem}\label{lem:IdealEquality}
	Let $\tilde{\Phi} = \bigcup_{i=1}^n \tilde{\cS}_i$ be a union of linear subspaces of $\Re^{D+1}$, and let $[\tilde{\Phi}] = \bigcup_{i=1}^n [\tilde{\cS}_i]$ be the corresponding projective variety of $\Pr^D$. Then $\I_{\tilde{\Phi},k} = \I_{[\tilde{\Phi}],k}$, i.e., a degree-$k$ homogeneous polynomial vanishes on $\tilde{\Phi}$ if and only if it vanishes on $[\tilde{\Phi}]$.
\end{lem}

\noindent As a Corollary of Theorem \ref{thm:ProjectiveClosureHomogenization} and Lemma \ref{lem:IdealEquality}, we obtain the key result of this Section, which we will use in Section \ref{subsection:ProofDense}.
\begin{prp}\label{prp:AffineUnionIdealHomogenization}
	Let $\Psi = \bigcup_{i=1}^n \A_i$ be a union of affine subspaces of $\Re^D$. Let $\tilde{\Phi} = \bigcup_{i=1}^n \tilde{\cS}_i$ be the union of linear subspaces of $\Re^{D+1}$ associated to $\Psi$ under the embedding $\phi_0$ of \eqref{eq:phi0}. Then $\I_{\tilde{\Phi}}$ is the homogenization of $\I_{\Psi}$. 
\end{prp} 
\section{Proofs of Main Theorems} \label{section:Proof}

\subsection{Proof of Theorem \ref{thm:MainTheoremDense}} \label{subsection:ProofDense}
	
	$(\Rightarrow)$ Suppose that $\X$ is in general position in $\Psi$. We need to show that $\tilde{\X}$ is in general position in $\tilde{\Phi}$. In view of Proposition \ref{prp:LinearGeneralPositionEquivalence}, and the fact that $\I_{\tilde{\Phi},n} \subset \I_{\tilde{\X},n}$, it is sufficient to show that $\I_{\tilde{\Phi},n} \supset \I_{\tilde{\X},n}$. To that end, let $P$ be a homogeneous polynomial of degree $n$ in $\Re[x_0,x_1,\dots,x_D]$ that vanishes on the points $\tilde{\X}$, i.e., $P \in \I_{\tilde{X},n}$. Then for every point $\tilde{\boldsymbol{\alpha}}= (1,\alpha_1,\dots,\alpha_D)$ of $\tilde{\X}$, we have
	\begin{align}
	P(\tilde{\boldsymbol{\alpha}})=P(1,\alpha_1,\dots,\alpha_D)=P_{(d)}(\alpha_1,\dots,\alpha_D) = 0,
	\end{align} that is, the dehomogenization $P_{(d)}$ of $P$ vanishes on all points of $\X$, i.e., $P_{(d)} \in \I_{\X}$.
	Now there are two possibilities: either $P_{(d)}$ has degree $n$, in which case $P=\left(P_{(d)}\right)^{(h)} $, or $P_{(d)}$ has degree strictly less than $n$, say $ n-k, \, k \ge 1$, in which case $P=x_0^k\left(P_{(d)}\right)^{(h)}$.	
	If $P_{(d)}$ has total degree $n$, by the general position assumption on $\X$, $P_{(d)}$ must vanish on $\Psi$. Then by Proposition \ref{prp:AffineUnionIdealHomogenization}, $\left(P_{(d)}\right)^{(h)} \in \I_{\Phi,n}$, and so $P \in \I_{\Phi,n}$.	
	If $\deg P_{(d)} = n-k, \, k \ge 1$, suppose we can find a linear form $G=\tilde{\boldsymbol{\zeta}}^\transpose \tilde{x}$, that does not vanish on any of the $\tilde{\cS}_i, \, i \in [n]$, and it is not divisible by $x_0$. Then $G_{(d)}$ will have degree $1$ and will not vanish on any of the $\A_i, \, i \in [n]$. Also $\left(G_{(d)}\right)^k P_{(d)}$ has degree $n$ and vanishes on $\X$. Since $\X$ is in general position in $\Psi$, we will have that $\left(G_{(d)}\right)^k P_{(d)}$ vanishes on $\Psi$. Then by Proposition \ref{prp:AffineUnionIdealHomogenization}, 
	$G^k \left(P_{(d)}\right)^{(h)} \in \I_{\tilde{\Phi},n}$. Since $\I_{\tilde{\Phi}} = \bigcap_{i=1}^n \I_{\tilde{\cS}_i}$ we must have that $G^k \left(P_{(d)}\right)^{(h)} \in \I_{\tilde{\cS}_i}, \, \forall i \in [n]$. Since $\I_{\tilde{\cS}_i}$ is a prime ideal (Proposition \ref{prp:LinearSubspaceIdeal}) and $G \not\in \I_{\tilde{\cS}_i}$, it must be the case that 
	$\left(P_{(d)}\right)^{(h)} \in \I_{\tilde{\cS}_i}, \, \forall i \in [n]$, i.e., $\left(P_{(d)}\right)^{(h)} \in \I_{\tilde{\Phi}}$. But $P = x_0^k \left(P_{(d)}\right)^{(h)}$, which shows that $P \in \I_{\Phi,n}$.
	
	It remains to be shown that there exists a linear form $G$ non-divisible by $x_0$, that does not vanish on any of the $\tilde{\cS}_i$. Suppose this is not true; thus if $G = \b^\transpose x + \alpha x_0$ is a linear form non-divisible by $x_0$, i.e., $\b \neq \0$, then $G$ must vanish on some $\tilde{\cS}_i$. In particular, for any non-zero vector $\b$ of $\Re^D$, $\b^\transpose x = \b^\transpose x + 0 x_0$ must vanish on some $\tilde{\cS}_i$. Recall from Section \ref{subsection:traditional}, that if $\u_{i1},\dots,\u_{id_i}$ is a basis for $\cS_i$, the linear part of $\A_i = \cS_i + \bmu_i$, then 
	\begin{align}
	\begin{bmatrix}
	1 & 0 & \cdots & 0 \\
	\bmu_i & \u_{i1} & \cdots & \u_{id_i}  \end{bmatrix}
	\end{align} is a basis for $\tilde{\cS}_i$. Since $\b^\transpose x$ vanishes on $\tilde{\cS}_i$, it must vanish on each basis vector of $\tilde{\cS}_i$. In particular, $\b^\transpose \u_{i1} = \cdots = \b^\transpose \u_{id_i} = 0$, which implies that 
	the linear form $\b^\transpose x$, now viewed as a function on $\Re^D$, vanishes on $\cS_i$, i.e., $\b^\transpose x \in \I_{\cS_i}$. To summarize, we have shown that for every $\0 \neq \b \in \Re^D$, there exists an $i \in [n]$ such that $\b^\transpose x \in \I_{\cS_i}$. Taking $\b$ equal to the standard vector $\e_1$ of $\Re^D$, we see that the linear form $x_1$ must vanish on some $\cS_i$, and similarly for the linear forms $x_2,\dots,x_D$. This in turn means that the ideal $\mathfrak{m}:=(x_1,\dots,x_D)$ generated by the linear forms $x_1,\dots,x_D$, must lie in the union $\bigcup_{i=1}^n \I_{\cS_i}$. But it is known from Proposition 1.11(i) in \cite{AtiyahMacDonald-1994}, that if an ideal $\mathfrak{a}$ lies in the union of finitely many prime ideals, then the $\mathfrak{a}$ must lie in one of these prime ideals. Applying this result to our case, we see that, since the $\I_{\cS_i}$ are prime ideals,  $\mathfrak{m} \subset \I_{\cS_i}$ for some $i \in [n]$. But this says that for any vector in $\cS_i$ all of its coordinates must be zero, i.e., $\cS_i = 0$, which violates the assumption $d_i >0, , \forall i \in [n]$. This contradiction proves the existence of our linear form $G$.
		 
	$(\Leftarrow)$ Now suppose that $\tilde{\X}$ is in general position in $\tilde{\Phi}$. We need to show that $\X$ is in general position in $\Psi$. To that end, let $p$ be a vanishing polynomial of $\Psi$ of degree $n$, then clearly $p \in \I_{\X}$. Conversely, let $p \in \I_{\X}$ of degree $n$. Then for each point $\boldsymbol{\alpha} \in \X$ 
	\begin{align}
	0&=p(\boldsymbol{\alpha})=p(\alpha_1,\dots,\alpha_D) \nonumber \\	
	&= p^{(h)}(1,\alpha_1,\dots,\alpha_D) 
	= p^{(h)}(\tilde{\boldsymbol{\alpha}}),
	\end{align} i.e., the homogenization $p^{(h)}$ vanishes on $\tilde{\X}$. By hypothesis $\tilde{\X}$ is in general position in $\tilde{\Phi}$, hence $p^{(h)} \in \I_{\tilde{\Phi},n}$. Then by Proposition \ref{prp:AffineUnionIdealHomogenization}, the dehomogenization of $p^{(h)}$ must vanish on $\Psi$. But notice that $\left(p^{(h)}\right)_{(d)} = p$, and so $p$ vanishes on $\Psi$.	

\subsection{Proof of Theorem \ref{thm:MainTransversal}} \label{subsection:ProofTransversal}
	If $\b_{i1},\dots,\b_{ic_i}$ is an orthonormal basis of $\cS_i^\perp$, $\tilde{\b}_{i1},\dots,\tilde{\b}_{ic_i}$, with	
	\small
	\begin{align}
	\tilde{\b}_{ij_i}:= [\b_{ij_i}^\transpose -\b_{ij_i}^\transpose \B_i \a_i]^\transpose = [\b_{ij_i}^\transpose -\a_i(j_i)]^\transpose, \, j_i \in[c_i],
	\end{align} \normalsize is a basis for $\tilde{\cS}_i^\perp$. Suppose that $\tilde{\Phi}$ is not transversal. Then there is some index set $\Jf \subset [n]$, say $\Jf=\left\{1,\dots,\ell\right\}, \, \ell \le n$, such that
		\begin{align}
	&\rank(\tilde{\B}_{\Jf})< \min\left\{D+1,\sum_{i \in \Jf} c_i \right\}, \label{eq:TransversalityViolated} \, \, \, \text{where}\\ 
	&\tilde{\B}_{\Jf}:= \left[\tilde{\B}_1, \dots, \tilde{\B}_{\ell}\right], \, \, \, \tilde{\B}_{i}:= [\tilde{\b}_{i1},\dots,\tilde{\b}_{ic_i}].
	\end{align} Since $\Phi$ is transversal, either $\rank(\B_{\Jf})=D$ or $\rank(\B_{\Jf})= \sum_{i \in \Jf} c_i=:c_{\Jf}$. 
	Suppose the latter condition is true, then $c_{\Jf} \le D$. Then all columns of $\B_{\Jf}$ are linearly independent, which implies that the same will be true for the columns of $\tilde{\B}_{\Jf}$, and so 
	$\rank(\tilde{\B}_{\Jf}) = c_{\Jf}$. Since by hypothesis
	$c_{\Jf} \le D$, we have that
	\begin{align}
	\codim \bigcap_{i \in \Jf} \tilde{\cS}_i = \rank(\tilde{\B}_{\Jf}) = \min\left\{D+1, c_{\Jf} \right\},
	\end{align} which contradicts \eqref{eq:TransversalityViolated}. Consequently, it must be the case that $\rank(\B_{\Jf}) =  D< c_{\Jf}$. Since $\B_{\Jf}$ is a submatrix of $\tilde{\B}_{\Jf}$, we must have $\rank \tilde{\B}_{\Jf} \ge D$. On the other hand, because of \eqref{eq:TransversalityViolated} we must have 
	$\rank(\tilde{\B}_{\Jf}) \le D$, and so $\rank(\tilde{\B}_{\Jf}) = D$. Now, $\tilde{\B}_{\Jf}$ is a $(D+1)\times c_{\Jf}$ matrix, with its minimal dimension being $(D+1)$.
Since its rank is $D$, it must be the case that all $(D+1)\times(D+1)$ minors of $\tilde{\B}_{\Jf}$ vanish. Taking the Laplace expansion of each such minor along its last row (the only place where the variables $\a_1,\dots,\a_n$ appear), we see that it is in fact a linear polynomial in the $\a_i$, with the coefficients being $D \times D$ minors of $\B_{\Jf}$. Since $\B_{\Jf}$ has rank $D$, at least one of its $D \times D$ minors is nonzero, and so at least one of the $(D+1)\times(D+1)$ minors of $\tilde{\B}_{\Jf}$ is a nonzero polynomial. Thus the algebraic variety $\W_{\Jf}$ defined by the vanishing of all the minors of $\tilde{\B}_{\Jf}$ is a proper subset of the space $\prod_{i=1}^n \Re^{c_i}$ that parametrizes the $\left\{\a_i\right\}$. Finally, $\tilde{\Phi}$ is transversal if and only if the $\left\{\a_i\right\}$ lie in the (open and dense) complement of the proper algebraic variety $\W := \bigcup_{\Jf \subset [n]} \W_{\Jf}$ of $\prod_{i=1}^n \Re^{c_i}$.

\section{Conclusions}
We established in a rigorous fashion the correctness of ASC in the case of affine subspaces. Using the technical framework of algebraic geometry, we showed that the embedding of points lying in general position inside a union of affine subspaces preserves the general position. Moreover, the embedding of a transversal union of affine subspaces will almost surely give a transversal union of linear subspaces. Future research will 
aim at
optimal realizations of the embedding in the presence of noise, analyzing SSC for affine subspaces, and reducing the complexity of ASC.

\bibliographystyle{IEEEtran}

\bibliography{../../../../CVS_files/biblio/sparse,../../../../CVS_files/biblio/learning,../../../../CVS_files/biblio/vidal,../../../../CVS_files/biblio/math,../../../../CVS_files/biblio/geometry,../../../../CVS_files/biblio/segmentation,../../../../CVS_files/biblio/vision,../../../../CVS_files/biblio/dataset,../../../../CVS_files/biblio/matrixcompletion}

\end{document}